\documentclass[11pt]{article}

\usepackage{fullpage,times}
\usepackage[square, numbers]{natbib}
\usepackage[utf8]{inputenc} % allow utf-8 input
\usepackage[T1]{fontenc}    % use 8-bit T1 fonts
\usepackage{hyperref}       % hyperlinks
\usepackage{url}            % simple URL typesetting
\usepackage{booktabs}       % professional-quality tables
\usepackage{amsfonts}       % blackboard math symbols
\usepackage{nicefrac}       % compact symbols for 1/2, etc.
\usepackage{microtype}      % microtypography

\usepackage{setspace}
\usepackage{bibspacing}
\usepackage{enumitem}
\usepackage{amsmath,amssymb,amsthm}

\usepackage{mathdots}
\def\R {{\mathbb{R}}}
\def\bbe {\mathbb{E}}
\def\S {{\mathcal{S}}}
\def\E {\mathcal{E}}
\def\V {\mathcal{V}}
\def\F {\mathcal{F}}
\def\I {\mathcal{I}}
\def\W {\mathcal{W}}
\def\O {\mathcal{O}}
\def\tO {\tilde{\mathcal{O}}}

\theoremstyle{plain}

\newtheorem{theorem}{Theorem}
\newtheorem{corollary}[theorem]{Corollary}
\newtheorem{lemma}[theorem]{Lemma}

\newtheorem{definition}[theorem]{Definition}

\title{Communication Lower Bounds for Distributed Convex Optimization: Partition Data on Features}

\author{Zihao Chen \\
Zhiyuan College \\
Shanghai Jiao Tong University \\
z.h.chen@sjtu.edu.cn \\
\and
Luo Luo \\
Department of Computer\\Science and Engineering \\
Shanghai Jiao Tong University \\
ricky@sjtu.edu.cn \\
\and
Zhihua Zhang \\
School of Mathematical Sciences\\
Peking University \\
zhzhang@math.pku.edu.cn \\
}
\date{}

\begin{document}
\maketitle

\begin{abstract}
 Recently, there has been an increasing interest in designing distributed convex optimization algorithms under the setting where the data matrix is partitioned on features. Algorithms under this setting sometimes have many advantages over those under the setting where data is partitioned on samples, especially when the number of features is huge. Therefore, it is important to understand the inherent limitations of these optimization problems. In this paper, with certain restrictions on the communication allowed in the procedures, we develop tight lower bounds on communication rounds for a broad class of non-incremental algorithms under this setting. We also provide a lower bound on communication rounds for a class of (randomized) incremental algorithms.
\end{abstract}

\section{Introduction}
In this paper, we consider the following distributed convex optimization problem over $m$ machines:
\begin{displaymath}
  \min_{w \in \R^{d}} f(w; \theta).
\end{displaymath}
Each machine knows the form of $f$ but only has some partial information of $\theta$. In particular, we mainly consider the case of \textit{empirical risk minimization} (ERM) problems. Let $A \in \mathbb{R}^{n \times d}$ be a matrix containing $n$ data samples with $d$ features and $A_{j:}$ be the $j$-th row of the data matrix (corresponding to the $j$-th data sample). Then $f$ has the form:
\begin{equation}
\label{function}
  f(w) = \frac{1}{n} \sum_{j = 1}^{n} \phi(w, A_{j:}),
\end{equation}
where $\phi$ is some kind of convex loss.

%With the increasing size of the data matrix, both the number of data samples $n$ and the feature dimension $d$ can be enormously large. A single machine is thus not able to store all the data in its memory or storage. Therefore, we need to partition and distribute data to several machines and use distributed optimization algorithms that rely on inter-machine communication. 

In the past few years, many distributed optimizations algorithms have been proposed. Many of them are under the setting where data is partitioned on samples, i.e. each machine stores a subset of the data matrix $A$'s rows \cite{DISCO,zhang2012communication,balcan2012distributed,boyd2011distributed,yang2013trading,jlee2015distributed,jaggi2014communication,ma2015adding}. Meanwhile, as the dimension $d$ can be enormously large, there has been an increasing interest in designing algorithms with the setting where the data is partitioned on features, i.e., each machine stores a subset of $A$'s columns \cite{richtarik2013distributed,Distributed-block-coordinate-descent,DISCO-F,1312.5302,lee2015distributed}. Compared with algorithms under the sample partition setting, these algorithms have relatively less communication cost when $d$ is large. In addition, as there is often no master machine in these algorithms, they tend to have more balanced workload on each machine, which is also an important factor affecting the performance of a distributed algorithm.

As communication is usually the bottleneck of distributed optimization algorithms, it is important to understand the fundamental limits of distributed optimization algorithms, i.e., how much communication an algorithm must need to reach an $\epsilon$-approximation of the minimum value. Studying fundamental communication complexity of the distributed computing without any assumption is quite hard, letting alone continuous optimization. As for optimization, one alternative way is to derive lower bounds on communication rounds. The number of communication rounds is also an important metric as in many algorithms, faster machines often need to pause and wait for slower ones before they communicate, which can be a huge waste of time. Recently, putting some restrictions on communication allowed in each iteration, \citeauthor{Ohad} (2015) developed lower bounds on communication rounds for a class of distributed optimization algorithms under the sample partition setting, and these lower bounds can be matched by some existing algorithms. However, optimization problems under the feature partition setting are still not well understood.

Considering the increasing interest and importance of designing distributed optimization algorithms under the feature partition setting, in our paper, we develop tight lower bounds on communication rounds for a broad class of distributed optimization algorithms under this setting. Our results can provide deeper understanding and insights for designing optimization algorithms under this setting. To define the class of algorithms, we put some constraints on the form and amount of communication in each round, while keeping restrictions mild and applying to many distributed algorithms. We summarize our contributions as follows:

\begin{itemize}
  \item For the class of smooth and $\lambda$-strongly convex functions with condition number $\kappa$, we develop a tight lower bound of $\Omega\left(\sqrt{\kappa}\log(\frac{\lambda {\|w^* - w_0 \|}}{ \epsilon})\right)$, which can be matched by a straightforward distributed version of accelerated gradient decent \cite{nesterov2013introductory} and also DISCO-F for quadratics \cite{DISCO-F}, an variant of DISCO \cite{DISCO} under the feature partition setting.

  \item For the class of smooth and (non-strongly) convex functions with Lipschitz smooth parameter $L$, we develop a tight lower bound of $\Omega\left(\sqrt{\frac{L}{\epsilon}} \|w^* - w_0\|\right)$, which is also matched by the distributed accelerated gradient decent.

  \item  By slightly modifying the definitions of the algorithms, we define a class of incremental/stochastic algorithms under the feature partition setting and develop a lower bound of $\Omega\left(\left(\sqrt{n\kappa} + n \right) \log ( \frac{{\|w^{*} - w_0\|} \lambda}{\epsilon} )\right)$ for $\lambda$-strongly convex functions with condition number $\kappa$.
\end{itemize}

\noindent \textbf{Related Work} The most revelant work should be \cite{Ohad}, which studied lower bounds under the sample partition setting, and provided tight lower bounds on communication rounds for convex smooth optimization after putting some mild restrictions. More recently, \citeauthor{jlee2015distributed} (2015) provided a lower bound for another class of algorithms under that setting. Both these work as well as ours are based on some techniques used in non-distributed optimization lower bound analysis \cite{nesterov2013introductory,lan2015optimal}.

\section{Notations and Preliminaries}
We use $\| \cdot \|$ as Euclidean norm throughout the paper. For a vector $w \in \R^d$, we denote $w(i)$ as the $i$-th coordinate of the vector $w \in \R^{d}$. We let the coordinate index set $[d]$ be partitioned into $m$ disjoint sets $\S_1, \S_2, \dots, \S_m$ with $\sum_{i = 1}^{m} d_i = d$ and $\S_j = \{ k \in [d] \big| \sum_{i < j} d_i < k \le \sum_{i \le j} d_i \}$ for $j = 1,2,\dots,m$. For a vector $w \in \mathbb{R}^d$, denote $w^{[j]}$ as a vector in $\mathbb{R}^{d_j}$ which equals to the segment of $w$ on coordinates $\S_j$. For a set of vectors $\V \subseteq \mathbb{R}^d$, we define $\V^{[j]} = \{ v^{[j]} \big| v \in \V \}\subseteq \R^{d_j}$. Then we define $f_j^{\prime}(x):= \left.{\frac{\partial f(w)}{\partial w^{[j]}}}\right|_{w = x}$ and $f_{ij}^{\prime\prime}(x) := \left.{\frac{\partial^{2} f(w)}{\partial w^{[i]} \partial w^{[j]}}}\right|_{w = x}$.

Through the whole paper, we use \textit{partition-on-sample} and \textit{partition-on-feature} to describe distributed algorithms or communication lower bounds under the settings where data is partitioned on samples and features respectively.

Then we list several preliminaries:
\begin{enumerate}
\item \textbf{Lipschitz continuity}. A function $h$ is called Lipschitz continuous with constant $L$ if
$$\forall x,y \in \textrm{dom }h \quad \|h(x) - h(y)\| \le L \|x - y\|.$$

\item \textbf{Lipschitz smooth and strongly convex}. A function $h$ is called $L$-smooth and $\lambda$-strongly convex if
$$\frac{\lambda}{2}{\|x - y\|}^2 \le h(y) - h(x) - {(x-y)}^T \nabla h(y) \le \frac{L}{2}{\|x - y\|}^2$$

\item \textbf{Communication operations}. Here we list some common MapReduce types of communication operations \cite{dean2008mapreduce} in an abstract level:
    \begin{enumerate}
      \item \textit{One-to-all broadcast}. One-to-all broadcast is an operation that one machine sends identical data to all other machines.
      \item \textit{All-to-all broadcast}. All-to-all broadcast is an operation that each machines performs a one-to-all broadcast simultaneously.
      \item \textit{Reduce}. Consider the setting where each processor has $p$ units of data, Reduce is an operation that combines the data items piece-wise (using some associative operator, such as addition or min), and make the result available at a target machine. For example, if each machine owns an $\mathbb{R}^d$ vector, then computing the average of the vectors needs a Reduce operation of an $\mathbb{R}^d$ vector.
      \item \textit{ReduceAll}. A ReduceAll operation can be viewed as a combination of a Reduce operation and a one-to-all broadcast operation.
    \end{enumerate}
\end{enumerate}

\section{Definitions and Framework}
In this section we first describe a family of distributed optimization algorithms using $m$ machines, and then modify it to get a family of incremental algorithms. At the beginning, the feature coordinates are partitioned into $m$ sets and each machine owns the data columns corresponding to its coordinate set. We model the algorithms as iterative processes in multiple rounds and each round consists of a  \textit{computation phase} followed by a \textit{communication phase}. For each machine we define a feasible set and during the computation phase, each machine can do some ``cheap" communication and add some vectors to it. During the \textit{communication phase} each machine can broadcast some limited number of points to all other machines. We also assume the communication operations are the common operations like broadcast, Reduce and ReduceAll. 

\subsection{Non-incremental Algorithm Family}
Here we define the non-incremental algorithm class in a formal way:

\begin{definition}[partition-on-feature distributed optimization algorithm family $\mathcal{F}^{\lambda, L}$ ]
We say an algorithm $\mathcal{A}$ for solving (\ref{function}) with $m$ machines belongs to the family $\mathcal{F}^{\lambda, L}$ of distributed optimization algorithms for minimizing $L$-smooth and $\lambda$-strongly convex functions ($\lambda = 0$ for non-strongly-convex functions) with the form (\ref{function}) if the data is partitioned as follows:
\begin{itemize}
  \item  Let the coordinate index set $[d]$ be partitioned into $m$ disjoint sets $\S_1, \S_2, \dots, \S_m$ with $\sum_{i = 1}^{m} d_i = d$ and $\S_j = \{ k \in [d] \big| \sum_{j < i} d_i < k \le \sum_{j \le i} d_i \}$ for $j = 1,2,\dots,m$. The data matrix $A \in \mathbb{R}^{n \times d}$ is partitioned column-wise as $A=[A_1, \dots, A_m]$, where $A_j$ consists of columns $i \in \S_j$. Each machine $j$ stores $A_j$.
\end{itemize}
and the machines do the following operations in each round:
\begin{enumerate}
  \item 
      %Each machine $j$ maintains a local feasible set of vectors $\W_j \subseteq \mathbb{R}^{d_j}$ initialized to be $\W_j^{(0)} = \{0\}$, and a record of $\W_i$'s for $i \ne j$. Denote $\W_j^{(k)}$ as the $j$'s feasible set $\W_j$ in the $k$-th round.

      For each machine $j$, define a feasible set of vectors $\W_j \subseteq \R^{d_j}$ initialized to be $\W_j^{(0)} = \{0\}$. Denote $\W_j^{(k)}$ as machine j's feasible set $\W_j$ in the $k$-th round.
  \item \textbf{Assumption on feasible sets.} In the $k$-th round, initially $\W_j^{(k)} = \W_j^{(k-1)}$. Then for a constant number of times, each machine $j$ can add any $w_j$ to $\W_j^{(k)}$ if $w_j$ satisfies
      \begin{gather}\label{assumption1}
 w_j \in \textrm{span} \Big\{u_j, ~ f'_j(u), ~(f_{jj}''(u) + D) v_j, ~f^{\prime\prime}_{ji}(u)v_i \ \Big| %\nonumber \\
  ~{u}^T = [{u_1}^T, \dots, {u_m}^T], \nonumber \\~
        u_j \in \W_j^{(k)}, ~v_j \in \W_j^{(k)},%\nonumber \\
         ~u_i \in \W_i^{(k-1)},~ v_i \in \W_i^{(k-1)}, ~i \ne j ,~
        D \textrm{ is diagonal} \Big\}.
      \end{gather}
  \item \textbf{Computation phase.} Machines do local computations and can perform no more than some constant times of Reduce/ReduceAll operations of an $\R^n$ vector or constant during the computation phase.
  \item \textbf{Communication phase.} At the end of each round, each machine $j$ can simultaneously broadcast no more than constant number of $\R^{d_j}$ vectors.
  \item The final output after $R$ rounds is $w^T = [w_1^T,\dots, w_m^T]$ satisfies $w_j \in \W_j^{(R)}$.
\end{enumerate}
\end{definition}

%Similarly, we can define the class of algorithms $\mathcal{F}_{nonsmooth}^{\lambda, L}$ minimizing nonsmooth convex Lipschitz continuous functions by slightly modifying the assumption on $\W_j$'s, which is provided in the appendix.

We have several remarks on the defined class of algorithms:

\begin{itemize}
  \item As described above, during the iterations to find $w^{*}$,  machine $j$ can only do updates on coordinates $\S_j$. This restriction is natural because machine $j$ does not have much information about function $f$ on other coordinates. Actually, almost all existing partition-on-feature algorithms satisfy this restriction \cite{richtarik2013distributed,Distributed-block-coordinate-descent,DISCO-F,1312.5302,lee2015distributed}.

  \item Similar to \cite{Ohad}, we use $\W_j$ to define the restriction on the updates. This is not an explicit part of algorithms and machines do not necessarily need to store it, nor do they need to evaluate the points every time they add to the feasible sets. Although in most algorithms belonging to this family, each machine broadcasts the points it has added to the feasible set and store what other machines have broadcast in the communication phase (a simple example is the straightforward distributed implementation of gradient decent), we choose not to define feasible sets as physical sets that machines need to store  to keep our results general. 

  \item The assumption on the updates is mild and it applies to many partition-on-feature algorithms. It allows machines to perform preconditioning using local second order information or use partial gradient to update. It also allows to compute and utilize global second order information, since the span we define includes the $f^{\prime\prime}_{ji}(u)v_i$ term. Besides, we emphasize that putting such structural assumptions is necessary. Even if we could develop some assumption-free communication lower bounds, they might be too weak and have a large gap with upper bounds provided by existing algorithms, thus becoming less meaningful and cannot provide deeper understanding or insights for designing algorithms.

  \item During the computation phase, we assume each local machines can perform unbounded amount of computation and only limited amount of communication. Note that this part of communication is a must in many partition-on-feature algorithms, usually due to the need of computing partial gradients $f'_j(w)$. However, for some common loss functions $\phi$, such as (regularized) squared loss, logistic loss and squared hinge loss, computing $f'_j(w)$ for all $j \in [m]$ in total only needs a ReduceAll operation of an $\mathbb{R}^{n}$ vector \cite{richtarik2013distributed}. In some gradient (or partial gradient) based algorithms \cite{richtarik2013distributed,Distributed-block-coordinate-descent}, communication to compute partial gradients are the only need of communication in the computation phase. Besides, some algorithms like DISCO-F \cite{DISCO-F} need to compute $(\nabla f(w)u)^{[j]}$. For loss functions like squared loss, logistic loss and squared hinge loss, it only requires the same amount of communication as computing partial gradients, i.e. a ReduceAll operation of an $\mathbb{R}^{n}$ vector \cite{DISCO-F}.

  \item Here we summarize the total communication allowed in each round. We use $\tilde{\mathcal{O}}$ to denote asymptotic bounds hiding constants and factors logarithmic in the required accuracy of the solution. During the computation phase, each machine can do constant times of ReduceAll operations of $\tO(n)$ bits. During the communication phase, each machine $j$ can broadcast $\tO(d_j)$ bits, this can be viewed as performing constant times of ReduceAll operation of an $\mathbb{R}^{d}$ vector. Therefore, the total communication allowed in each round is no more than ReduceAll operations of $\tO(n + d)$ bits. The amount of communication allowed in our partition-on-feature algorithm class is relatively small, compared with the partition-on-sample algorithm class described in \cite{Ohad}, which allows $\tO(md)$ bits of one-to-all broadcast in each round. This is due to the partition-on-feature algorithms' advantage on communication cost.

  \item We also emphasize that the communication allowed in the entire round is moderate. On one hand, if we only allow too little communication then the information exchange between machines is not enough to perform efficient optimization. As a result, hardly no practical distributed algorithm could satisfy the requirement, which will diminish the generality of our results. On the other hand, if we allow too much communication in each round, our assumption on the feasible sets can be too strong. For example if we allow enough communication for all machines to broadcast their entire local data, then the machines only need one communication round to find out a solution up to any accuracy $\epsilon$.

  \item Our assumption that $\W_j$ is initialized as $\{0\}$ is merely for convenience and we just need to shift the function via $\bar{f}(w) = f(w + w_0)$ for another starting point $w_0$.

\end{itemize}

\subsection{Incremental Algorithm Family}
To define the class of incremental/stochastic algorithms $\I^{\lambda, L}$ under the feature partition setting, we slightly modify the definition of $\F^{\lambda, L}$ by replacing assumption on feasible set (\ref{assumption1}) with the following while keeping the rest unchanged: 

    \textbf{Assumption on feasible sets for $\I^{\lambda, L}$.} In the $k$-th round, initially $\W_j^{(k)} = \W_j^{(k-1)}$. Next for a constant number of times, each machine $j$ chooses $g(w) := \phi(w, A_{l:})$ for some (possibly random) $l$ and adds any $w_j$ to $\W_j^{(k)}$ if $w_j$ satisfies
      \begin{gather}\label{assumption2}
 w_j \in \textrm{span} \Big\{u_j, ~ g'_j(u), ~(g_{jj}''(u) + D) v_j, ~g^{\prime\prime}_{ji}(u)v_i \ \Big| %\nonumber \\
  ~{u}^T = [{u_1}^T, \dots, {u_m}^T], \nonumber \\~
        u_j \in \W_j^{(k)}, ~v_j \in \W_j^{(k)},%\nonumber \\
         ~u_i \in \W_i^{(k-1)},~ v_i \in \W_i^{(k-1)}, ~i \ne j ,~
        D \textrm{ is diagonal} \Big\}.
      \end{gather}

\section{Main Results}
In this section, we present our main theorems followed by some discussions on their implications.

First, we present a lower bound on communication rounds for algorithms in $\mathcal{F}^{ \lambda, L}$:

\begin{theorem}\label{thm1}
  For any number $m$ of machines, any constants $\lambda , L, \epsilon > 0$, and any distributed optimization algorithm $\mathcal{A} \in \mathcal{F}^{ \lambda, L}$, there exists a $\lambda$-strongly convex and $L$-smooth function $f(w)$ with condition number $\kappa:= \frac{L}{\lambda}$ over $\mathbb{R}^d$ such that if $w^{*} = \arg \min_{w \in \R^{d}} f(w)$, then the number of communication rounds to obtain $\hat{w}$ satisfying $f(\hat{w}) - f(w^{*}) \le \epsilon$ is at least
    \begin{equation}
     \Omega \left( \sqrt{\kappa} \log \left( \frac{{\|w^{*}\|} \lambda}{\epsilon}\right) \right)
    \end{equation}
    for sufficiently large $d$.
\end{theorem}

Similarly, we have a lower bound for algorithms in $\mathcal{F}^{0, L}$ for minimizing smooth convex functions:
\begin{theorem}\label{thm2}
  For any number $m$ of machines, any constants $L > 0, \epsilon > 0$, and any distributed optimization algorithm $\mathcal{A} \in \mathcal{F}^{0, L}$, there exists a $L$-smooth convex function $f(w)$ over $\mathbb{R}^d$ such that if $w^{*} = \arg \min_{w \in \R^{d}} f(w)$, then the number of communication rounds to obtain $\hat{w}$ satisfying $f(\hat{w}) - f(w^{*}) \le \epsilon$ is at least
  \begin{equation}
    \Omega\left(\sqrt{\frac{L}{\epsilon}} \|w^*\|\right) 
  \end{equation}
  for sufficiently large $d$.

\end{theorem}

Then we contrast our lower bound for smooth strongly convex functions with some existing algorithms and the upper bounds provided by their convergence rate. The comparisons indicate that our lower bounds are tight. For some common loss functions, this can be matched by a straightforward distributed implementation of accelerated gradient decent \cite{nesterov2013introductory} and it is easy to verify that it satisfies our definition: let all machines compute their own partial gradients and aggregate to form a gradient. This straightforward distributed version of accelerated gradient decent achieves a round complexity of $O\left(\sqrt{\kappa} \log (\frac{{\|w^{*} - w_0\|} \lambda}{\epsilon})\right)$, which matches our lower bound exactly. Similarly, our lower bound on smooth (non-strongly) convex functions can also be matched by the distributed accelerated gradient decent.

Recall that our definition of the algorithm class includes some types of distributed second order algorithms, for example DISCO-F \cite{DISCO-F}. The number of communication rounds DISCO-F needs to minimize general quadratic functions is $\O\left(\sqrt{\kappa} \log(\frac{\|w^* - w_0\|}{\epsilon})\right)$ , which also matches our bound with respect to $\kappa$ and $\epsilon$. This indicates that distributed second order algorithms may not achieve a faster convergence rate than first order ones if only linear communication is allowed.

Next we present a lower bound on communication rounds for algorithms in $\I^{\lambda, L}$:

    \begin{theorem}\label{thm6}
     For any number $m$ of machines, any constants $\lambda , L, \epsilon > 0$, and any distributed optimization algorithm $\mathcal{A} \in \mathcal{I}^{ \lambda, L}$, there exists a $\lambda$-strongly convex and $L$-smooth function $f(w)$ with condition number $\kappa:= \frac{L}{\lambda}$ over $\mathbb{R}^d$ such that if $w^{*} = \arg \min_{w \in \R^{d}} f(w)$, then the number of communication rounds to obtain $\hat{w}$ satisfying $\bbe \left[ f(\hat{w}) - f(w^{*}) \right] \le \epsilon$ is at least
    \begin{equation}
  \Omega \left( \left(\sqrt{n\kappa} + n \right) \log \left( \frac{{\|w^{*}\|} \lambda}{\epsilon} \right) \right)
    \end{equation}
    for sufficiently large $d$.
  \end{theorem}

\section{Proof of Main Results}
In this section, we provide proofs of Theorem \ref{thm1} and Theorem \ref{thm6}. The proof framework of these theorems are based on \cite{nesterov2013introductory}. As Theorem \ref{thm2} can be obtained by replacing \cite[Lemma 2.1.3]{nesterov2013introductory} with Corollary \ref{lemma4} (see below) in the proof of \cite[Theorem 2.1.6]{nesterov2013introductory}, we will not discuss it here for simplicity.

\subsection{Proof of Theorem 2}
 The idea is to construct a ``hard" function so that all algorithms in the class we defined could not optimize well in a small number of rounds: in each round only one of the machines can do a constant steps of ``progress" while other machines stay ``trapped" (see Lemma \ref{lemma2}). 

%In this section we first give the proof of a lemma that can be important for all the three theorems. The proof framework is based on \cite{nesterov2013introductory}.
%In this section we give the proof of Theorem \ref{thm1} based on Lemma \ref{lemma4}. Theorem \ref{thm2} can be proved in an almost the same way with Lemma \ref{lemma4}. The idea is to construct a `hard' function so that all algorithms in the class we defined could not optimize well in a small number of rounds: in each round only one of the machines can do a constant steps of `progress' while other machines stay `trapped' (see Lemma \ref{lemma4}). Our proof is inspired by \cite{nesterov2013introductory}.

Without loss of generality, we assume that in every round, each machine only add one vector to its feasible set and the bound does not change asymptotically.

First, we construct the following function like \cite{lan2015optimal}
    \begin{equation}
        f(w)=\frac{\lambda(\kappa-1)}{4} \big[\frac{1}{2}w^TAw -  \langle e_1,w \rangle\big] + \frac{\lambda}{2}\|w\|^2,
    \end{equation}
    where $A$ is a tridiagonal matrix in $\R^{d \times d}$ with the form
    \begin{equation*}
        A = \left[\begin{matrix}
            ~~~2 &   -1 &  ~~~0 &  \cdots & ~~~0 & ~~~0 & ~~~0 \\
              -1 & ~~~2 &    -1 &  \cdots & ~~~0 & ~~~0 & ~~~0 \\
            ~~~0 &   -1 &  ~~~2 &  \cdots & ~~~0 & ~~~0 & ~~~0 \\
             ~~\cdots & ~~\cdots & ~~\cdots & \cdots & ~~\cdots & ~~\cdots & ~~\cdots \\
            ~~~0 & ~~~0 &  ~~~0 & \cdots &   -1 & ~~~2 &   -1 \\
            ~~~0 & ~~~0 &  ~~~0 & \cdots & ~~~0 &   -1 & \frac{\sqrt{\kappa}+3}{\sqrt{\kappa}+1} \\
        \end{matrix}\right].
    \end{equation*}
    It is easy to verify that $f(w)$ is $\lambda$-strongly convex with condition number $\kappa$.

     After $K$ rounds of iteration, let $\E_{t,d}:=\{x\in\R^d \big| x(i)=0 , \ t+1 \leq i \leq d\}$ and $\W^{(K)} := \{ {[w_1^T, \dots, w_m^T]}^T \big| w_j \in \W_j^{(K)}, j = 1,\dots,m \}$. 

     Then we have the following lemma:
    \begin{lemma}\label{lemma2}
         If $\W^{(K)} \subseteq \E_{K,d}$ for some $K\leq d-1$, then we have $\W^{(K+1)} \subseteq \E_{K+1,d}$.
    \end{lemma}
  \begin{proof}
    First we recall the assumption on $\W_j$'s in (\ref{assumption1}):
            \begin{gather*}
        w_j \in \textrm{span} \Big\{u_j, ~ f'_j(u), ~(f_{jj}''(u) + D) v_j, ~f^{\prime\prime}_{ji}(u)v_i \ \Big| %\nonumber \\
         ~{u}^T = [{u_1}^T, \dots, {u_m}^T],\nonumber \\ ~
        u_j \in \W_j^{(k)}, ~v_j \in \W_j^{(k)},%\nonumber \\
         ~u_i \in \W_i^{(k-1)},~ v_i \in \W_i^{(k-1)}, ~i \ne j, ~
        D \textrm{ is diagonal} \Big\}.
            \end{gather*}

            We just need to prove that for any vector $u, ~v \in \W^{(K)} \subseteq \E_{K,d}$,
            \begin{equation*}
              u_j, f'_j(u), ~(f''_{jj}(u)+D)v_j, ~ f''_{ji}(u)v_i \in \E_{K+1,d}^{[j]}
            \end{equation*}

            For convenience of the proof, we partition $A$ as follow
            \begin{equation}
                 A = \left[\begin{matrix}
                    A_{11} & A_{12} & A_{13} \\
                    A_{21} & A_{22} & A_{23} \\
                    A_{31} & A_{32} & A_{33}
                \end{matrix}\right],
            \end{equation}
            where $A_{11}\in\mathbb{R}^{a\times a}$, $A_{22}\in\mathbb{R}^{b\times b}$, $A_{33}\in\R^{c\times c}$ and $a=\sum_{i<j}d_i$, $b=d_j$ , $c=\sum_{i>j}d_i$ . Let $x_1=[u_1^T, u_2^T, \dots, u_{j-1}^T]^T$, $x_2=u_j$, and $x_3=[u_{j+1}^T, \dots, u_m^T]^T$. \\
            Then we obtain
            \begin{equation*}
                f'_j(u) = \left\{ \begin{array}{ll}
                    \Big[\frac{\lambda(\kappa-1)}{2}A_{22} + \lambda I \Big]x_2 + %\\[0.2cm]
                     \frac{\lambda(\kappa-1)}{2}(A_{21}x_1+A_{23}x_3) & j\neq 1 \\[0.3cm]
                    \Big[\frac{\lambda(\kappa-1)}{2}A_{22} + \lambda I \Big]x_2 +  %\\[0.2cm]
                    \frac{\lambda(\kappa-1)}{2}(A_{21}x_1+A_{23}x_3) - \frac{\lambda(\kappa-1)}{4}e_1^{[1]} & j=1
                \end{array}\right.
            \end{equation*}
            and
            \begin{equation*}
              f''_{jj}(u) = \frac{\lambda(\kappa - 1)}{4} A_{22} + \lambda I.
            \end{equation*}
            As $f''_{jj}(u) + D$ is a tridiagonal matrix, we have
            \begin{equation}\label{eqn1}
              (f''_{jj}(u) + D) v_j \in \E_{K+1, d}^{[j]}.
            \end{equation}
            Using the fact that
            \begin{equation*}
                 A_{21} = \left[\begin{matrix}
                        0 & \cdots &  0  &  -1 \\
                        0 & \cdots &  0  &  0 \\
                    \cdots & \cdots &  \cdots  &  \cdots \\
                        0 & \cdots &  0  & 0
                 \end{matrix}\right] 
                 A_{23} = \left[\begin{matrix}
                        0 & 0 & \cdots  & 0 \\
                       \cdots & \cdots & \cdots & \cdots \\
                        0 & 0 & \cdots  &  0 \\
                          -1 & 0 & \cdots & 0
                 \end{matrix}\right],
            \end{equation*}

            we can rewrite $f'_j$ as
            \begin{equation*}
                f'_j(u) = \begin{cases}
                    \Big[\frac{\lambda(\kappa-1)}{2}A_{22} + \lambda I \Big]x_2 + %\\[0.2cm] 
                    \frac{\lambda(\kappa-1)}{2}[x_1(a) , 0, \dots , 0,  x_3(1)]^T,  & j\neq 1 
                    \\[0.3cm] 
                    \Big[\frac{\lambda(\kappa-1)}{2}A_{22} + \lambda I \Big]x_2 + %\\[0.2cm]
                     \frac{\lambda(\kappa-1)}{2}[\frac{1}{2} , 0, \dots , 0, x_3(1)]^T. &  j=1 
                \end{cases}
            \end{equation*}

            As $\frac{\lambda(\kappa-1)}{2}A_{22} + \lambda I$ is a tridiagonal matrix,  then
            \begin{equation*}
              \Big[\frac{\lambda(\kappa-1)}{2}A_{22} + \lambda I\Big] x_2 \in \E_{K+1, d}^{[j]}.
            \end{equation*}
            We can also obtain the following by some simple discussions on different cases:
            \begin{equation*}
              [\frac{1}{2} , 0, \dots , 0, x_3(1)]^T, ~ [x_1(a) , 0, \dots , 0,  x_3(1)]^T \in \E_{K+1, d}^{[j]}
            \end{equation*}
            Therefore, we conclude that
            \begin{equation}\label{eqn2}
              f'_j(w) \in \E_{K+1, d}^{[j]}.
            \end{equation}
            Besides, note that
            \begin{equation*}
                f^{\prime\prime}_{ji}(u) v_i = \begin{cases}
                    [0,\dots,0, -v_i(1)]^T & i=j+1 \\
                    [-v_i(d_i), 0, \dots, 0]^T & i=j-1 \\
                    [0, 0, \dots, 0]^T & \text{otherwise}
                \end{cases}
            \end{equation*}
            Similarly, using some simple discussions on different cases we obtain
            \begin{equation}\label{eqn3}
              f^{\prime\prime}_{ji}(u) v_i \in \E_{K+1, d}^{[j]}
            \end{equation}

            Hence when $\W^{(K)} \subseteq \E_{K,d}$, combining (\ref{eqn1}) (\ref{eqn2}) and (\ref{eqn3}) we have for all $u \in \W^{T}$ and $j$
            \begin{equation*}
                 u_j, ~f'_j(u), ~(f''_{jj}(u)+D)v_j, ~ f''_{ji}(u)v_i\in \E_{K+1,d}^{[j]}.
            \end{equation*}
            which implies the newly added $w_j$ satisfies
            \begin{equation*}
            w_j \in \E_{K+1,d}^{[j]}
            \end{equation*}
            Therefore, we prove that $\W^{(K+1)} \subseteq \E_{K+1,d}$.
    \end{proof}

    Applying Lemma \ref{lemma2} recursively we can get the following corollary:
    \begin{corollary}\label{lemma4}
      After $K \le d$ rounds, we have $\W^{(K)} \subseteq \E_{K,d}.$
    \end{corollary}

    With Corollary \ref{lemma4},  we now proceed to finish the proof of Theorem \ref{thm1}. First, we can find $w^*$ by the first order optimality condition
    \begin{equation*}
        f^{\prime}(w^*) = \Big(\frac{\lambda(\kappa-1)}{4}A + \lambda I \Big)w^* - \frac{\lambda(\kappa-1)}{4} = 0,
    \end{equation*}
    which implies
    \begin{equation*}
        \Big(A + \frac{4}{\kappa-1}I\Big)w^* = e_1.
    \end{equation*}
    The coordinate form of above equation is
    \begin{eqnarray*}
        2\frac{\kappa+1}{\kappa-1}w^*(1) - w^*(2)  &=&  1,  \\
        w^*(k+1) - 2\frac{\kappa+1}{\kappa-1}w^*(k) + w^*(k-1)  &=&  0,%\\ 
        2 \leq k \leq d-2, \\
        -x^*(d-1) + \left(\frac{4}{\kappa-1} + \frac{\sqrt{\kappa}+3}{\sqrt{\kappa}+1}\right)x^*(d) &=& 0,
    \end{eqnarray*}
    and let $q$ be the smallest root of the following equation
    \begin{equation*}
        q^2 - \frac{2\kappa+2}{\kappa-1}q + 1 = 0,
    \end{equation*}
    that is $q = \frac{\sqrt{\kappa}-1}{\sqrt{\kappa}+1}$. Then $w^*$ satisfies $w^*(i)=q^i$ for $1 \leq i\leq d$. Hence,
    \begin{eqnarray*}
        \|w^*\|^2  = \sum_{i=1}^{d} [w^*(i)]^2 = \sum_{i=1}^{d} q^{2i} = \frac{q^2(1-q^{2d})}{1-q^2}
    \end{eqnarray*}

    Let $w^{(k)}$ be any point in $\W^{(k)}$ after $k$ rounds iterations ($k \le d$), applying Corollary \ref{lemma4} we have
    \begin{eqnarray*}
            \|w^{(k)} - w^*\|^2
        &\geq& \sum_{i=k+1}^{d} [w^*(i)]^2 = \sum_{i=k+1}^{d} q^{2i}  \\
        &=& \frac{q^{2(k+1)}[1-q^{2(d-k+2)}]}{1-q^2} \\
        &=& \frac{1-q^{2(d-k+2)}}{1-q^{2d}}q^{2k}\|w^*\|^2 \\
        &\ge& \frac{1-q^4}{1-q^{2d}}q^{2k}\|w^*\|^2 \\
        &\ge& \frac{1-q}{1}q^{2k}\|w^*\|^2
    \end{eqnarray*}
    Combing the above inequality and the optimal condition of strongly-convex function, we obtain
    \begin{eqnarray*}
        f(w^{(k)}) - f(w^*)
        &\geq& \frac{\lambda}{2}  \|w^{(k)} - w^*\|^2 \nonumber \\
        &\geq& \frac{\lambda(1-q)}{2} q^{2k}\|w^*\|^2 \nonumber \\
        &=& \frac{\lambda}{2} \frac{2}{\sqrt{\kappa}+1} \left( \frac{\sqrt{\kappa}-1}{\sqrt{\kappa}+1} \right)^{2k} \|w^*\|^2 \nonumber \\
        &\ge& \frac{\lambda}{\sqrt{\kappa} + 1} \exp\left(-\frac{4k}{\sqrt{\kappa} + 1}\right)\|w^*\|^2.
    \end{eqnarray*}

    Thus if $f(w^{(k)}) - f(w^*) \le \epsilon$, then we have
    \begin{align*}
    k \geq ~ \frac{\sqrt{\kappa} - 1}{4} \log \left( \frac{\lambda \|w^*\|^2 }{(\sqrt{\kappa} + 1)\epsilon}\right) %\\ 
    = ~ \Omega \left( \sqrt{\kappa} \log \left( \frac{\lambda \|w^*\|}{\epsilon}\right) \right),
    \end{align*}
    which completes the proof.

\subsection{Proof of Theorem \ref{thm6}}
With a slight abuse of notation, we construct the following separable strongly convex function:
\begin{equation}
f(w) := \frac{1}{m} \sum_{j = 1}^{m} \phi_j(w_j), 
\end{equation}
where $w^T = [w_1^T, \dots, w_m^T]$ and $\phi_j(w_j)$ is also a separable strongly convex function with form of 
\begin{align}
\phi_j(w_j) = \sum_{i = 1}^{\frac{n}{m}} \bigg[ \frac{\lambda(\kappa-1)}{4} \left(\frac{1}{2}w_{j,i}^TA_{j,i}w_{j,i} -  \langle e_1,w_{j,i} \rangle\right) %\nonumber \\
       + \frac{\lambda}{2}\|w_{j,i}\|^2 \bigg], 
\end{align}
where $w_j^T = [w_{j,1}^T, \dots, w_{j,n}^T]$ and  $A_{j,i}$ is a tridiagonal matrix in $\R^{d_j \times d_j}$ given by 
\begin{equation*}
    A_{j,i} = \left[\begin{matrix}
            ~~~2 &   -1 &  ~~~0 &  \cdots & ~~~0 & ~~~0 & ~~~0 \\
              -1 & ~~~2 &    -1 &  \cdots & ~~~0 & ~~~0 & ~~~0 \\
            ~~~0 &   -1 &  ~~~2 &  \cdots & ~~~0 & ~~~0 & ~~~0 \\
             ~~\cdots & ~~\cdots & ~~\cdots & \cdots & ~~\cdots & ~~\cdots & ~~\cdots \\
            ~~~0 & ~~~0 &  ~~~0 & \cdots &   -1 & ~~~2 &   -1 \\
            ~~~0 & ~~~0 &  ~~~0 & \cdots & ~~~0 &   -1 & \frac{\sqrt{\kappa}+3}{\sqrt{\kappa}+1} \\
    \end{matrix}\right].
\end{equation*}

It is simple to veriy that $f$ is a $\lambda$-strongly convex function with condition number $\kappa$. Note that $f$ is a quadratic function with its coefficient matrix being a block diagonal matrix, in which each block matrix is tridiagonal. As $f$ is separable, the setting here can be viewed as each machine $j$ has all information (all data samples in all coordinates) of $f$'s component $\phi_j$, and all machines simultaneously do non-distributed incremental opitmizaion over their local data. To get the lower bound, note that in ``clever" algorithms, machine $j$ only chooses data samples among the ones corresponding to $\phi_j$ in each round, then we have 
    \begin{eqnarray*}
            &&\bbe[\|w^{(k)} - w^*\|^2] \\
        &\geq& \bbe\left[\sum_{j = 1}^{m} \|w_{j}^{(k)} - w_j^* \|^2\right] \\
        &\geq& \sum_{j = 1}^{m} \left[\frac{1}{2} \exp\left(-\frac{4k\sqrt{\kappa}}{n(\sqrt{\kappa} + 1)^2 - 4\sqrt{\kappa}}\right)\|w^*_{j}\|^2\right] \\
        &=& \frac{1}{2} \exp\left(-\frac{4k\sqrt{\kappa}}{n(\sqrt{\kappa} + 1)^2 - 4\sqrt{\kappa}}\right)\|w^*\|^2.
    \end{eqnarray*}
The second inequality is according to \cite[Theorem 3]{lan2015optimal} and Corollary \ref{lemma4}. 

Finally by \cite[Corollary 3]{lan2015optimal}, we get if $\bbe \left[ f(w^{(k)}) - f(w^*) \right] \le \epsilon$, then
\begin{equation}
    k \ge \Omega \left( \left(\sqrt{n\kappa} + n \right) \log \left( \frac{{\|w^{*}\|} \lambda}{\epsilon} \right) \right),
\end{equation}
for sufficiently large $d$.

\section{Conclusion}
In this paper we have defined two classes of distributed optimization algorithms under the setting where data is partitioned on features: one is a family of non-incremental algorithms and the other is incremental. We have presented tight lower bounds on communication rounds for non-incremental class of algorithms. We have also provided one lower bound for incremental class of algorithms but whether it is tight remains open. 

The tightness informs that one should break our definition when trying to design optimization algorithms with less communication rounds than existing algorithms. We also emphasize that our lower bounds are important as they can provide deeper understanding on the limits of some ideas or techniques used in distributed optimization algorithms, which may provide some insights for designing better algorithms.

To the best of our knowledge, this is the first work to study communication lower bounds for distributed optimization algorithms under the setting where data is partitioned on features.

% \section*{Acknowledgments}
% This work is supported by National Natural Science Foundation of China (No. 61572017).

\bibliography{sample}

\end{document}